%% file: main.tex
\documentclass{article}

\PassOptionsToPackage{numbers, compress}{natbib}

 \usepackage[main, final]{neurips_2026}

\usepackage[utf8]{inputenc} 
\usepackage[T1]{fontenc}    
\usepackage{hyperref}       
\usepackage{url}            
\usepackage{booktabs}       
\usepackage{amsfonts}       
\usepackage{nicefrac}       
\usepackage{microtype}      
\usepackage{xcolor}         
\usepackage{multirow}       
\usepackage{bbm}
\usepackage{amsmath}
\usepackage{amsthm}
\usepackage{graphicx}       
\newtheorem{theorem}{Theorem}
\newtheorem{lemma}{Lemma}
\usepackage{enumitem}

\usepackage{wrapfig}
\usepackage[most]{tcolorbox}
\newtcolorbox{promptbox}[1][]{
    enhanced,
    colback=gray!5,
    colframe=gray!55,
    boxrule=0.6pt,
    arc=1.5mm,
    left=3mm,
    right=3mm,
    top=3mm,
    bottom=3mm,
    fonttitle=\bfseries\normalsize\sffamily\color{gray!70!black},
    title=#1,
    breakable,
    fontupper=\ttfamily\normalsize,    
    before skip=10pt,
    after skip=10pt,
}

\usepackage{tabularx}
\usepackage{booktabs}
\newtcolorbox{casestudybox}[1][]{
    enhanced,
    breakable,
    colback=blue!2,
    colframe=blue!35!black,
    boxrule=0.45pt,
    arc=0.8mm,
    left=3mm,
    right=3mm,
    top=2.2mm,
    bottom=2.2mm,
    fonttitle=\bfseries\normalsize\sffamily\color{blue!35!black},
    title=#1,
    fontupper=\normalfont\normalsize,
    before skip=10pt,
    after skip=10pt,
}
\newcommand{\casefield}[1]{%
    \vspace{0.35em}
    \noindent\textbf{\sffamily #1.}
}
\newcommand{\modelchoice}[2]{%
    \vspace{0.45em}
    \noindent\textbf{\sffamily #1.} #2
}

\definecolor{mygreen}{RGB}{0, 128, 80}   
\definecolor{myred}{RGB}{200, 60, 60}    

\usepackage{subcaption}
\usepackage{algorithmic}
\usepackage{algorithm}
\usepackage[table]{xcolor}
\usepackage{booktabs}
\usepackage{multirow}
\usepackage{array}

\title{PAFO: Pareto Fairness Optimization for \\ Personalized Reward Modeling}

\author{%
  Xiaoyan Zhao$^{1}$\thanks{Equal contribution.}\hspace{3pt}, 
  Haoting Ni$^{2}$\footnotemark[1]\hspace{3pt}, 
  Yang Zhang$^{1}$, Chunyuan Zheng$^{3}$, Haoxuan Li$^{3}$, Fuli Feng$^{2}$ \\[4pt]
  $^1$National University of Singapore \\ $^2$University of Science and Technology of China \quad
  $^3$Peking University \\[2pt]
  \texttt{xiaoyanzhao.ai@gmail.com}
}

\begin{document}
\maketitle
\begin{abstract}
    Large language models (LLMs) increasingly rely on reward models to align their outputs with diverse user preferences. While personalized reward models aim to capture such heterogeneity, they are often trained on imbalanced user preference data and may therefore favor users whose preferences are more common in the training population. In this paper, we identify this failure mode as personalized reward bias, where reward modeling quality varies systematically with preference support rate. We formulate its mitigation as a Pareto fairness problem over group utilities, aiming to improve under-served users without degrading other user groups. To this end, we propose \textit{\textbf{PAFO}}, a Pareto fairness optimization framework for personalized reward modeling. PAFO first trains group-specialized reward models for majority and minority preference groups, then constructs conditional margin-level supervision to distill their heterogeneous preference boundaries into a single unified model. The resulting model uses group information only during training and requires no explicit group labels at inference time. Experiments on Personal-LLM and DSP show that PAFO improves both minority-group and majority-group accuracy while reducing user-level unfairness across multiple metrics, demonstrating its effectiveness for fairer LLM personalization.
    
\end{abstract}

\input{sections/1_intro}

\input{sections/2_problem_formulation}

\input{sections/3_method}

\input{sections/4_experimentv1}

\input{sections/5_related_workv1}

\input{sections/6_conclusionv1}

\bibliographystyle{unsrtnat}
\bibliography{refs}

\appendix

\input{sections/7_appendix}



\end{document}

%% file: sections/1_intro.tex
\section{Introduction}

\setlength{\columnsep}{4pt}
\begin{wrapfigure}{r}{0.33\linewidth}
    \vspace{-12pt}
    \includegraphics[width=0.9\linewidth]{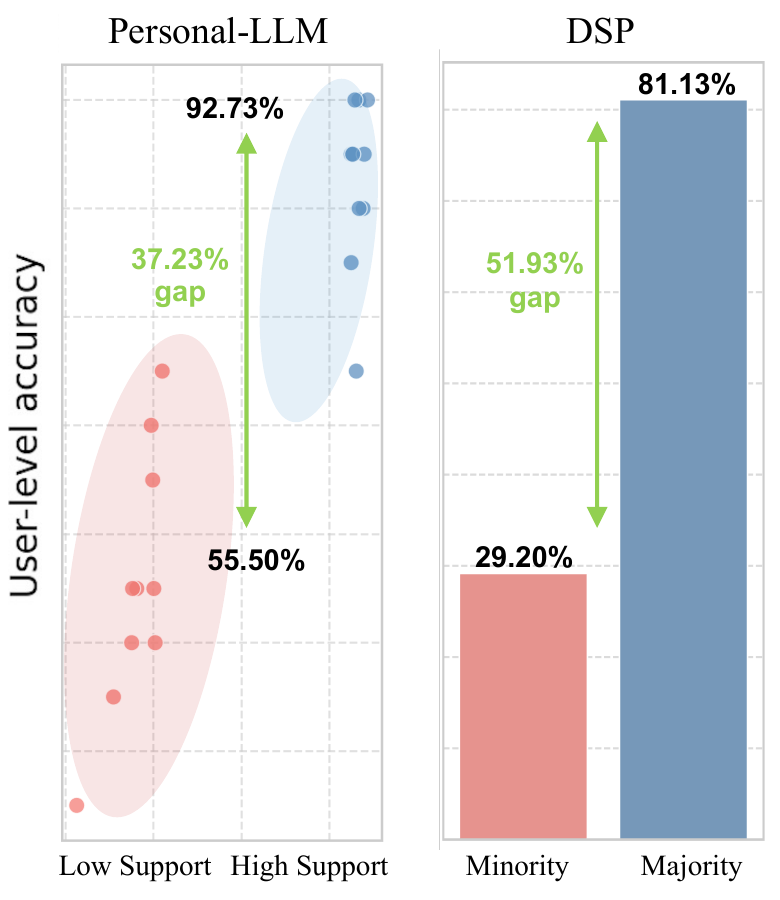}
    \caption{Minority users suffer substantially lower reward prediction accuracy than majority users, even under personalized RM.}
    \label{fig:unfairness}
    \vspace{-8pt}
\end{wrapfigure}
Reward models are central to the alignment of large language models, providing the scalar signals used to optimize policies in RLHF~\cite{ouyang2022training,stiennon2020learning,kim2024rethinking} and to select responses at test time~\cite{liu2025pairwise,hung2025reward}.
However, reward models are not neutral estimators of human judgment. Existing studies have shown that they can encode systematic biases~\cite{huang2024post,gallegos2024bias,singhal2023long,gallegos2024bias,hayes2024large}, such as length bias, sycophancy, and social bias. These findings suggest that the reward signal used for alignment may reflect artifacts of the training or annotation process rather than the preferences that users actually intend to express~\cite{lambert2025rewardbench,sharma2023towards}. 

Most existing analyses study reward model bias in the population-level setting~\cite{lambert2025rewardbench,xiao2024algorithmic}, where a single reward function is shared across all users~\cite{song2025towards} and bias is examined as a shared distortion or along predefined demographic axes. This perspective is insufficient for personalized alignment, where different users may naturally prefer different response styles, reasoning patterns, or interaction norms\citep{zhao2026nextquill,zhang2026reinforced,zhao2025steerx,qiu2025latent}. However, personalization addresses preference heterogeneity without addressing preference-support imbalance. When preference data are imbalanced, the training objective can still be dominated by users whose preferences are more common, leaving rare or less-supported preference patterns insufficiently represented~\citep{dong2025personalization}.

We refer to this failure mode as \textit{personalization unfairness} in reward modeling: a systematic disparity in preference accuracy across users, structured by how well a user's preference pattern is supported in the training distribution. Unlike conventional reward model biases, which manifest as a shared distortion across the population (e.g., length bias)~\citep{huang2024post} or along pre-specified demographic axes~\citep{song2025towards,PAFO}, personalization unfairness is \textit{support-structured}---its incidence on a given user is determined by the rarity of that user's preference pattern, a quantity that is neither annotated nor necessarily tied to explicit demographic attributes. Figure~\ref{fig:unfairness} illustrates this phenomenon empirically. On the Personal-LLM benchmark~\citep{zollo2024personalllm}, representative personalized reward modeling methods exhibit a strong dependence between per-user accuracy and per-user support rate; on the DSP benchmark~\citep{cheng2023everyone}, accuracy on minority preference styles is substantially lower than on majority styles. Therefore, personalization may inherit preference-support imbalance rather than eliminate it.

Mitigating personalization unfairness is challenging. Users with minority preferences provide few supervision signals, so their reward boundaries are hard to learn from data alone. Training a separate model per group would help, but separate models are unusable at inference in practice, where group membership is typically unknown. Single-model fixes such as reweighting or fairness regularization avoid this problem but tend to lift minority accuracy at the expense of majority accuracy, redistributing error rather than reducing it~\citep{ouyang2025towards}. We therefore formulate the objective as \textit{Pareto} improvement over preference groups: only no-regret gains for minorities count as solutions.
    
We propose PAFO (\textbf{PA}reto \textbf{F}airness \textbf{O}ptimization), a framework for personalized reward modeling. The core idea is to let a reward model first specialize on each preference group and then consolidate these group-specific abilities back into itself through self-distillation, so that group information is consumed only during training. Concretely, PAFO partitions data into minority-to-majority preference groups and finetunes a group-specialized reward model for each group, so that each minority's preference structure is captured from within-group signal rather than diluted across the full population. A conditional teacher then routes each preference pair to its group-specialized model and uses the resulting reward margin as the supervision target for self-distillation. The same model, now serving as the student, is trained to match these group-specific margins while taking only the user-conditioning signals available at inference as input. This design realizes Pareto improvement: minority groups thus benefit from dedicated specialization, and majority groups are equally served by their own specialist, so the student inherits an undiluted majority signal instead of one reweighted against minority groups. Group information, used only as a training-time scaffold, is no longer needed at deployment.

\textbf{Our main contributions can be summarized as follows:}
\begin{itemize}[leftmargin=*]
\item \textbf{Problem.} We identify and formalize \textit{personalization unfairness} in reward modeling, where per-user reward quality is governed by preference-support imbalance, distinct from prior reward model biases that affect the population uniformly or along demographic axes.
\item \textbf{Method.} We frame mitigation as Pareto improvement over preference groups and propose PAFO, which specializes a reward model on each group and consolidates these abilities back via conditional margin-level self-distillation, requiring group information only at training time.
\item \textbf{Results.} On Personal-LLM and DSP, PAFO improves minority-group accuracy, preserves majority-group accuracy, and reduces user-level unfairness, showing that Pareto improvement is achievable in personalized reward modeling.
\end{itemize}

%% file: sections/2_problem_formulation.tex
\section{Preliminary Analysis}

\paragraph{Problem Formulation.}
We study the personalized reward modeling task. Given a query $x$, user historical information $h$, and a candidate response $y$, a personalized reward model $r_{\theta}(x, h, y)$ assigns a scalar score to the response conditioned on both the query context and the user-specific information. 
The preference-labeled dataset is $\mathcal{D}=\{(s_i^A, s_i^B, \mathbbm{1}\{s_i^A \succ s_i^B\})\}_{i=1}^{N}$, where $s^A=[x, y^A]$ and $s^B=[x, y^B]$ denote two candidate responses under the same prompt $x$, and $\mathbbm{1}\{s_i^A \succ s_i^B\}$ denotes the user's preference. The standard training objective is the Bradley--Terry--Luce (BTL) loss \citep{bradley1952rank}
\begin{equation}
\mathcal{L}_{\mathrm{BTL}}(\theta) \;=\; -\,\mathbb{E}_{(s^A, s^B) \sim \mathcal{D}}\!\left[\log \sigma\!\left(r_\theta(s^A) - r_\theta(s^B)\right)\right],
\label{eq:btl}
\end{equation}
where we abbreviate $r_\theta(s) = r_\theta(x, h, y)$ for compactness and $\sigma(\cdot)$ is the sigmoid. Unlike the conventional setting that assumes a single latent reward function shared across all users, our central question is whether a personalized reward model can leverage user-specific information to improve modeling fidelity while \emph{simultaneously} avoiding systematic disparities across user groups when preferences are heterogeneous.
\begin{figure}[t]
    \centering
    \begin{subfigure}[t]{0.98\linewidth}
        \centering
        \includegraphics[width=0.93\linewidth]{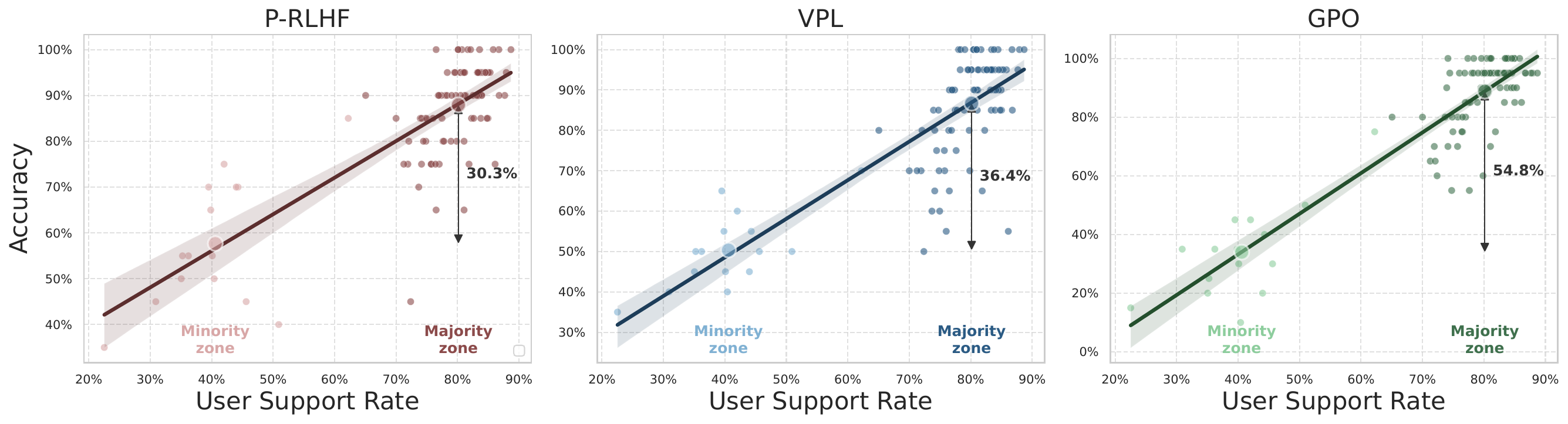}
        \caption{Personal-LLM: Accuracy increases with user support rate.}
        \label{fig:unfair_personal}
    \end{subfigure}
    \begin{subfigure}[t]{0.98\linewidth}
        \centering
        \includegraphics[width=0.93\linewidth]{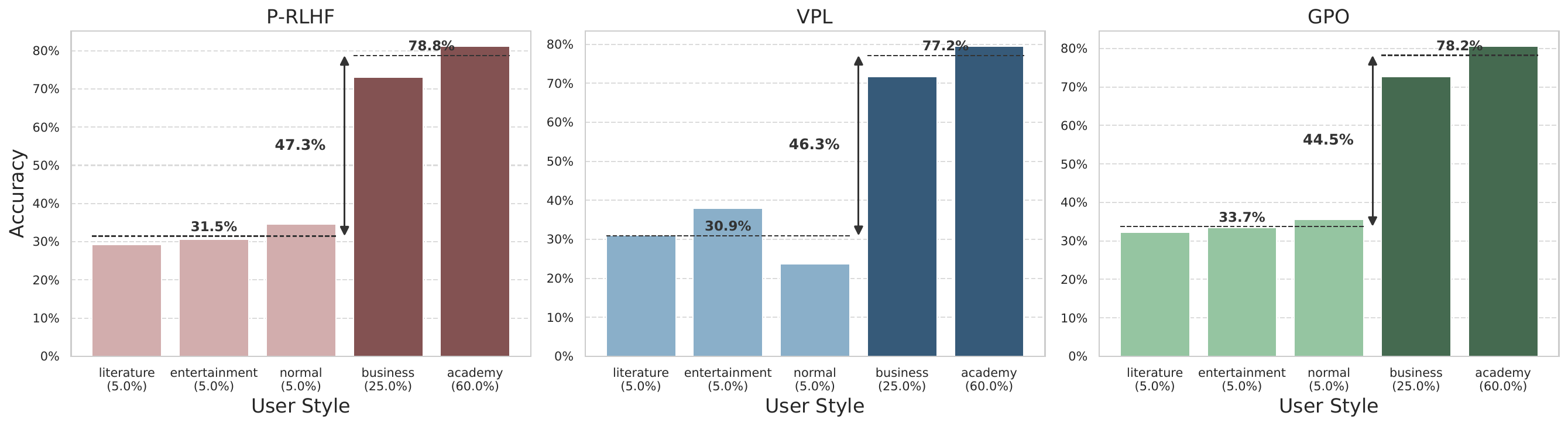}
        \caption{DSP: Minority preference styles receive substantially lower accuracy than majority styles.}
        \label{fig:unfair_dsp}
    \end{subfigure}
    \caption{Preliminary analysis on the Personal-LLM and DSP datasets across P-RLHF, VPL, and GPO methods, showing user-level unfairness in existing personalized RMs. }
    \label{fig:unfairness_problem_all}
    \vspace{-15pt}
\end{figure}

\paragraph{Analysis.}
We first examine whether existing personalized RMs provide comparable reward modeling quality across users with different levels of preference support. 
Figure~\ref{fig:unfairness_problem_all} reports user-level accuracy of three representative personalized reward modeling methods, P-RLHF, VPL, and GPO, on Personal-LLM and DSP. 
On Personal-LLM, accuracy increases strongly with user support rate, indicating that users whose preferences are less supported in the population receive lower accuracy. 
On DSP, minority preference styles obtain substantially lower accuracy than majority styles across all three methods. 
These results reveal a support-structured disparity in personalized reward modeling.

This phenomenon is notable because personalized RMs are designed to handle preference heterogeneity. 
A natural expectation is that conditioning the reward model on user-specific information $h$ should allow $r_\theta(x,h,y)$ to express different preferences for different users. 
However, Figure~\ref{fig:unfairness_problem_all} shows that personalization capacity alone does not guarantee equitable reward modeling quality: users with low-support preference patterns still receive systematically lower accuracy.

We attribute this phenomenon to preference-support imbalance in the training data and the aggregate nature of the BTL objective. 
First, user-specific signals differ in both quantity and quality. 
Some users provide stable preference pairs that align with widely shared patterns, while others provide fewer or noisier pairs that reflect less-supported preferences. 
Second, the BTL objective aggregates over the full dataset, so preferences that are better supported in the population contribute more frequent and more consistent gradient signals. 
As a result, even when the model architecture supports personalization, training can still be dominated by majority-aligned preference patterns, leaving low-support users under-supervised. 
Thus, personalization may inherit preference-support imbalance rather than eliminate it.

%% file: sections/3_method.tex
\section{Method}
\begin{figure}[t]
    \centering
    \includegraphics[width=0.95\linewidth]{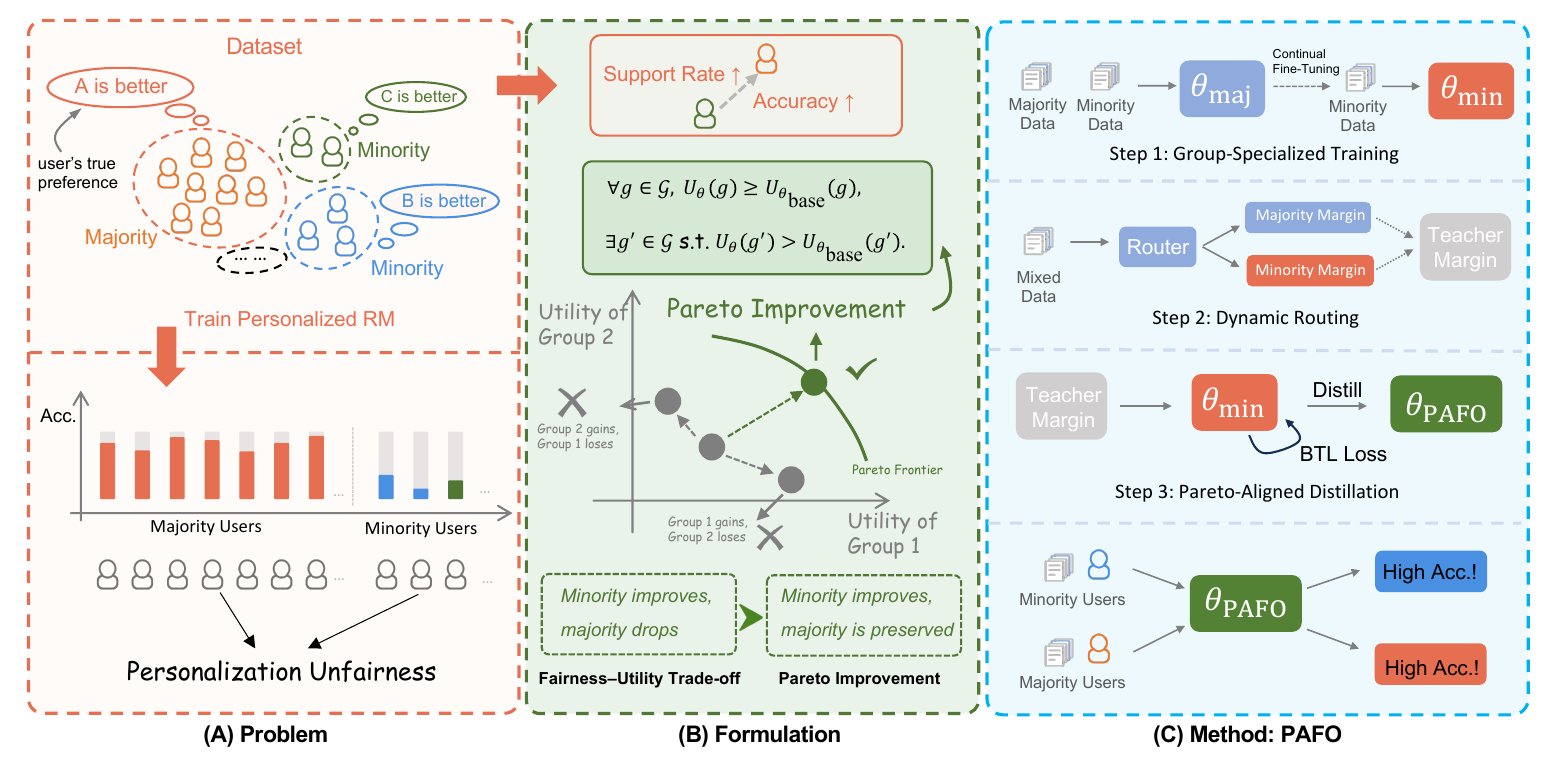}
    \caption{Overview of PAFO. (A) Imbalanced preference support leads to user-level unfairness in personalized reward modeling. (B) PAFO targets Pareto improvement, improving minority users without degrading majority users. (C) It trains group-specialized reward models and distills group-conditioned margin supervision into a single deployable personalized RM.}
    \label{fig:frame_work}
\end{figure}

In this section, to solve the imbalance problem, we first introduce the concept of Pareto fairness, then propose a principled distillation method and a lightweight model structure.

\subsection{Pareto Fairness}
We assume the users can be divided into $n$ groups, denoted as $\mathcal{G}=\{1,2,\ldots,n\}$, with each group showing a different preference. Without loss of generality, let the first group ($g=1$) be the majority group, and all remaining groups are minorities. Let $\mathcal{D}_g$ denote the data distribution of group $g \in \mathcal{G}$, and let $U_{\theta}(g)$ denote the expected utility of group $g$ under model parameters $\theta$. For example, when we train the reward model, we can take the negative BTL loss, i.e., $U_{\theta}(g)=\mathbb{E}_{(s^A, s^B) \sim \mathcal{D}_g}\!\left[\log \sigma\!\left(r_\theta(s^A) - r_\theta(s^B)\right)\right]$, as the utility.
Based on this formulation, we define a {Pareto improvement} as follows: relative to a baseline model $\theta_{\mathrm{base}}$, a new model $\theta$ does not decrease the utility of any group of interest and achieves a strict improvement for at least one group, which can be formulated as:
\begin{equation}
\label{pareto improvement}
    \forall g \in \mathcal{G},\; U_{\theta}(g) \ge U_{\theta_{\mathrm{base}}}(g),
    \qquad
    \exists g' \in \mathcal{G}\; \text{s.t.}\; U_{\theta}(g') > U_{\theta_{\mathrm{base}}}(g').
\end{equation}

Meanwhile, we define Pareto efficiency (also called Pareto fairness in our paper) as an ideal state in which the model cannot be further improved by any Pareto improvement, i.e., improving one group's utility must harm another group's utility. A naive way is to train models for each user group separately to solve this problem. However, it may significantly increase the costs of model storage, deployment, and inference. In addition, the group label of users is often unavailable at inference time. Therefore, we need a new training mechanism to train a unified model to make predictions in the following sections.

\subsection{Necessary Condition of Pareto Fairness}

Because sufficient and necessary conditions are difficult to derive, many previous works have focused only on the necessary conditions. We first introduce the following lemma:

\begin{lemma}[Necessary condition~\cite{sener2018multi}]
    Any solution $\theta$ that satisfies Pareto fairness must satisfy the following condition:
    \begin{equation}
    \exists~ \alpha_g>0,~g\in \{1, 2, \ldots,n\},~\sum_{g=1}^{n} \alpha_{g} = 1, ~\text{s.t.}\quad \sum_{g=1}^{n} \alpha_{g} \nabla_{\theta} U_{\theta}(g) = 0.
\end{equation}
\end{lemma}

Therefore, given the model parameter $\theta$, the problem of judging whether $\theta$ satisfies the necessary condition of Pareto fairness can be transferred to find $\alpha_t$ to minimize the following equation and examining whether the minimum of the following optimization problem is greater than $0$:
\begin{equation}
    \min_{\alpha_{1}, \dots, \alpha_{n}} \left\{ \left\| \sum_{g=1}^{n} \alpha_{t} \nabla_{\theta} U_{\theta}(g) \right\|_{2}^{2} \;\middle|\; \sum_{g=1}^{n} \alpha_{g} = 1, \alpha_{g} \ge 0, \; \forall g \right\}
    \label{eq:4}
\end{equation}

Previous study \cite{desideri2012multiple} shows that either the solution to this optimization problem is 0, resulting in a point that satisfies the necessary conditions, or the solution gives a descent direction that improves the utility of all groups. That is to say, we have a guaranteed model update direction once obtaining the weights. However, in the reward modeling setting of large language models, directly computing the explicit weight is usually infeasible. On the one hand, it is impractical to find an explicit solution for Equation \ref{eq:4} when $n > 2$. On the other hand, explicitly solving for the corresponding weights at every training step would introduce high computational cost and optimization instability.
Therefore, rather than directly solving for this explicit direction, we design a trainable approximation method to solve this problem.

\subsection{Framework Overview}

In this section, we provide a brief overview of the proposed Pareto Fairness Optimization (PAFO) framework for personalized reward modeling to approximate the weight. The core intuition is that a Pareto-aligned update direction can be achieved by dynamically assigning higher optimization weights to groups that fall behind an implicit ideal target, until achieving the balance, i.e., all groups have similar performance. Overall, PAFO consists of three stages. First, we construct reward models in each group to mimic the best possible performance. Second, we randomly choose a minority group to initialize the parameters, then further finetune the model based on the naive BTL loss and distillation loss to achieve the Pareto fairness.

For the model structure, as shown in Figure \ref{fig:frame_work}, we jointly use two types of personalized signals.
The first is a parameterized personalized representation based on the user ID: inspired by prior work\citep{li2024personalized}, we introduce a lightweight user model that maps each user identifier to a learnable user embedding, and prepend this embedding to the very beginning of the model input so as to explicitly encode user preferences and guide the model to form personalized reward judgments.
The second is contextual personalization based on user historical behavior: for the current query $x$, we additionally select multiple historical interaction samples of the same user and prepend them before $x$ as explicit historical context for the model.
Together, these two types of signals constitute the personalized conditional input of the reward model.

\subsection{Group-Specialized Training}

Since the group labels are available in the training stage, we can train group-specialized reward models in each group. Specifically, given a training sample $d_i = (x_i,h_i,y_i^w,y_i^l, g_i)$, where \(x_i\) denotes the input query, \(h_i\) denotes the user history, \(y_i^w\) and \(y_i^l\) denote the preferred response and rejected response, and $g_i$ denotes the user's group, respectively. The reward model \(r_\theta\) assigns scores to candidate responses, and the reward margin is defined as $m_{\theta}(i) = r_\theta(x_i,h_i,y_i^w) - r_\theta(x_i,h_i,y_i^l)$. We optimize the reward model with the standard Bradley--Terry loss:
\begin{equation}
    \mathcal{L}_{\mathrm{BTL}}(\theta)
    = - \mathbb{E}_{i\sim\mathcal{D}}
    \left[
    \log \sigma(m_{\theta}(i))
    \right],
\end{equation}
where \(\sigma(\cdot)\) is the sigmoid function. Since the majority group usually dominates in the training data, a personalized reward model trained directly on the full dataset tends to learn the preference structure of the majority group. 
Therefore, we regard the reward model trained on the mixed data as a majority-oriented reward model, denoted by $r_{\theta_{1}}$, and the corresponding margin as $m_{\theta_1}$. 
This model can capture high-support preference patterns relatively well, but is insufficient for modeling low-support preference users.

To further enhance the model's ability to capture the preference structure of the minority group, we continue fine-tuning \(r_{\theta_{1}}\) on minority-group data \(\mathcal{D}^g_{\mathrm{min}}\), obtaining a minority-specialized reward model:
\begin{equation}
    \theta_g
    =
    \arg\min_{\theta}
    \mathcal{L}_{\mathrm{BTL}}^{\mathrm{min}}(\theta),
    \text{ where }
    \mathcal{L}_{\mathrm{BTL}}^{\mathrm{min}}(\theta)
    =
    -
    \mathbb{E}_{i\sim\mathcal{D}_{\mathrm{min}}^g}
    \left[
    \log \sigma(m_{\theta}(i))
    \right],
\end{equation}
where $g \in \{2,\ldots,n\}$. The minority-specialized model \(r_{\theta_g}\) can capture the preference patterns of low-support users more adequately. Until now, we have obtained $n$ models, representing the best possible performance on $n$ groups, respectively.

\subsection{Distillation Objective}

After finetuning the models, we aim to dynamically assign optimization weights to groups based on the performance to train a unified model by distillation. Different from the traditional single-teacher distillation paradigm, we do not simply use the output of a single model as the supervision signal. 
Instead, we propose to construct a conditional teacher based on group priors at the reward margin level. 
Specifically, for any training sample, we define the final teacher margin as $m_T(x_i,h_i,y_i^w,y_i^l)=m_{\theta_g}$ if $g_i=g$, with $g \in \{1,2,\ldots,n\}$.

To protect the utility in the minority group, we randomly choose $\theta_g$ with $g \in \{2,\ldots,n\}$ to initialize the student, i.e., $
    \theta_S \leftarrow \theta_{g}.
$ For any training sample $d_i$, the student model produces the corresponding reward margin $m_S(d_i)$. To enable the student to learn the group preference boundary defined by the teacher, we impose a soft-label distillation constraint at the level of pairwise preference probability. Specifically, we map the teacher margin and the student margin into pairwise preference probabilities through the sigmoid function $q_T(i)=\sigma(m_T(d_i)),\ q_S(i)=\sigma(m_S(d_i))$.
We then define the margin-level distillation loss as
\begin{equation}
    \mathcal{L}_{\mathrm{distill}}
    =
    \mathbb{E}_{d_i\sim\mathcal{D}}
    \left[
    \mathrm{CE}
    \left(
    q_T(i), q_S(i)
    \right)
    \right],
\end{equation}
where \(\mathrm{CE}(\cdot,\cdot)\) denotes the cross-entropy loss. Meanwhile, to prevent the student from drifting away from the true preference-pair distribution during distillation, we retain the standard BTL loss as hard-label supervision. Our final optimization objective is to train a unified model using the following loss:
\begin{equation}
    \label{final loss}
    \mathcal{L}_{PAFO}
    =
    \alpha \cdot \mathcal{L}_{BTL}
    +
    ( 1-\alpha ) \cdot \mathcal{L}_{\mathrm{distill}}
\end{equation}
where \(\alpha\) is a hyperparameter controlling the strength of distillation.

\subsection{Theoretical Results}

In Section 3.2, we show that explicit Pareto optimization requires solving an optimization problem, which is impractical for large language models. In this section, we provide a theoretical analysis to show how the proposed distillation method approximates the optimal Pareto-aligned optimization. First, we introduce some mild assumptions below:

\textbf{Assumption 1} (Local Smoothness). \textit{The utility functions $U_g(\theta)$ are continuously differentiable and $L$-smooth for all $g \in \mathcal{G}$. Furthermore, the sample-level margin gradients are bounded, such that $\|\nabla_{\theta} m_{\theta}(i)\| \le C_\phi$ for all data instances.}

The most important thing we need to show is that our method can approximate the ground-truth update direction.

\begin{theorem}[Update direction]
    {The negative gradient of the distillation loss $v = - \nabla_{\theta_S} \mathcal{L}_{distill}$ is exactly composed of an adaptive conic combination of the true utility gradients plus a covariance term:}
\begin{equation}
    v = \sum_{g=1}^n \lambda_g \nabla U_{\theta_S}(g) + \mathcal{E}_{cov}, \quad \text{with } \lambda_g \ge 0,
\end{equation}
{where $\lambda_g = \mathbb{E}_{\mathcal{D}_g}[\alpha_g(i)]$, $\alpha_g(i) = \frac{\sigma(m_g(i)) - \sigma(m_S(i))}{1 - \sigma(m_S(i))}$, $\mathcal{E}_{cov} = \sum_{g=1}^n e_g$, and $e_g = \text{Cov}_{\mathcal{D}_g}(\alpha_g(i), (1 - \sigma(m_S(i))) \nabla_{\theta_S} m_S(i))$.}
\end{theorem}

The proof is in Appendix \ref{app:proof}. Building upon this decomposition, we can directly derive the following theorem to ensure a large enough cosine similarity between our update direction and the optimal direction when the covariance is small.

\begin{theorem}[Similarity between ours and optimal]
Let $d^*(\theta_S)$ be the Pareto optimal direction, with weight $(\alpha_1,\alpha_2,\ldots,\alpha_g)$ solved from Equation \ref{eq:4}. Let $v_{cone} = \sum_{g=1}^n \lambda_g \nabla U_{\theta_S}(g)$, $\mu = \cos(v_{cone}, d^*(\theta_S))$, and $\nu = \frac{\|\mathcal{E}_{cov}\|}{\|v_{cone}\|}$. As long as the current parameter $\theta_S$ is not at a Pareto stationary point, we have $\mu > 0$. Furthermore, we have
\begin{equation}
    \cos(v, d^*(\theta_S)) \ge \frac{\mu - \nu}{1 + \nu}
\end{equation}
\end{theorem}

A critical question arises when $\mu$ is close to zero. However, PAFO prevents orthogonal degeneration ($\mu \to 0$) through a dynamic negative feedback loop. Since both the Pareto optimum direction and our directions assign higher weights to underperforming groups to prevent utility degradation, if the update trajectory temporarily diverges from $d^*$, the student's prediction margin $\sigma(m_S)$ on that group deteriorates. Consequently, the implicit weight vector $(\lambda_1,\lambda_2,\ldots,\lambda_g)$ and the optimal Pareto weight vector $(\alpha_1,\alpha_2,\ldots,\alpha_g)$ exhibit strong positive correlation throughout the trajectory.

%% file: sections/4_experimentv1.tex
\section{Experiments}

\subsection{Experimental Settings}

\textbf{Datasets.} We use two datasets, \textit{Personal-LLM}~\cite{zollo2024personalllm} and \textit{DSP}~\cite{cheng2023everyone}, both adapted to expose support-imbalanced user populations. On Personal-LLM, the bottom 15\% of users by support rate form the minority. On DSP, each user is assigned one of five preference types, with three minority types each held by 5\% of users and two majority types held by 25\% and 60\%. Dataset construction, support-rate definitions, and group statistics are in Appendix~\ref{appendix:data-details}--\ref{appendix:databias}.

\input{figures/main_results}

\paragraph{Compared Methods.} We compare PAFO with six baselines covering three categories. 1) \textit{Standard reward modeling}: \textbf{Bradley-Terry-Luce (BTL)} \citep{bradley1952rank} adds a linear scoring head on the base LLM, the standard RLHF reward modeling approach. 2) \textit{Personalized reward modeling}: \textbf{Variational Preference Learning (VPL)}\citep{poddar2024personalizing} learns latent user representations via a variational autoencoder; \textbf{Group Preference Optimization (GPO)}\citep{zhao2023group} learns user representations through a Transformer with meta-learning; \textbf{Personalized-RLHF (P-RLHF)}\citep{zollo2024personalllm} injects user embeddings into the base LLM's input representations through a learnable user model. 3) \textit{Fairness interventions}: \textbf{Reweighting} upweights minority-group samples during training; \textbf{Regularization (Reg)} adds a regularizer on user-level reward distributions to reduce cross-user disparity. The implementation details for all methods are in Appendix~\ref{app:baseline_details}.

\paragraph{Evaluation Metrics.} We evaluate both utility and fairness. For utility, we report \textit{Overall Accuracy} (Acc.), \textit{Minority Accuracy}(Min Acc.), and \textit{Majority Accuracy} (Maj Acc.); the latter two reveal whether minority gains come at the majority's expense. For fairness, we report three metrics on the per-user Accuracy distribution: \textit{Coefficient of Variation (CV)} for the relative dispersion of user Acc.uracies, \textit{Gini coefficient (GINI)} for inequality, and \textit{Accuracy--Support Rate Slope (Slope)} from a linear fit of user Accuracy on support rate, where a larger slope indicates stronger bias toward mainstream users. Utility and fairness must be read jointly: utility gains may come from sacrificing minorities, while fairness gains may come from degrading overall utility. Full formulas are in Appendix~\ref{app:metrics}.

\subsection{Main Results}

We compare PAFO with the baselines on both datasets, with results reported in Table~\ref{tab:main_results}. PAFO achieves Pareto-style improvement over both groups simultaneously: Min Acc. rises by +18.3 on Personal-LLM and +43.0 on DSP over BTL, while Maj Acc. is improved on Personal-LLM and essentially preserved on DSP. PAFO is best on every metric on Personal-LLM, and leads on Acc., Min Acc., CV, and Gini on DSP, with AS-Slope virtually tied with Reweight. Notably, AS-Slope drops sharply on both datasets even as Min Acc. rises and Maj Acc. holds, indicating that PAFO breaks the dependency between user Accuracy and preference popularity.

The baselines reveal why this improvement is non-trivial. Personalization alone does not produce it: VPL and GPO degrade Min Acc. on Personal-LLM relative to the standard BTL baseline, and even P-RLHF---PAFO's own personalization backbone---falls notably short of PAFO on minorities (70.33 vs.\ 75.00 on Personal-LLM, 69.67 vs.\ 74.47 on DSP), confirming that user conditioning does not automatically absorb support imbalance. Existing Fairness interventions trade rather than inherently improve: Compared to their personalization backbone (P-RLHF), Reweight on DSP lifts Min Acc. only at a clear cost to Maj Acc.; while Reg secures the highest Maj Acc. on DSP at the price of leaving Min Acc. well below PAFO's. PAFO is the only method that substantially improves minorities while matching or exceeding its own backbone on majorities, validating that Pareto-style improvement is achievable in personalized reward modeling under a single deployable model.

\subsection{In-depth Analyses}
\input{figures/ablation}
\paragraph{Ablation Studies.} Table~\ref{tab:ablation} ablates PAFO's components on top of the Base model, which is trained on all data and inherits the majority bias (effectively serving as PAFO's majority specialist). We have four main results: 1) A minority specialist alone trades majorities for minorities. "Min Only" continues training Base on minority data, lifting Min Acc. dramatically (+23.2 / +18.8) but collapsing Maj Acc. (-16.6 / -37.0)---a textbook Pareto failure\footnote{Notably, PAFO's task is not to match the Min Specialist on minorities, but to deliver as much of its minority benefit as possible without surrendering the majority.}; 2) Non-conditional distillation fails to lift minorities. "Maj$\rightarrow$Min" distills Base into the minority specialist without conditional routing. Maj Acc. recovers, but Min Acc. barely moves (+1.4 / -0.5)---majority preference structure is precisely what minorities do not share; 3) SFT cannot substitute for margin-level distillation. "Min + SFT" replaces margin distillation with continued SFT. Performance is uneven (+4.8 on Personal-LLM, -3.3 on DSP), showing that response-level imitation cannot consolidate group-specific preference structure; 4) In contrast, PAFO improves both groups on both datasets, confirming that its three components---group-specialized modeling, conditional routing, and margin-level distillation---are jointly necessary.

\begin{figure}[h]
    \centering
    \includegraphics[width=0.98\linewidth]{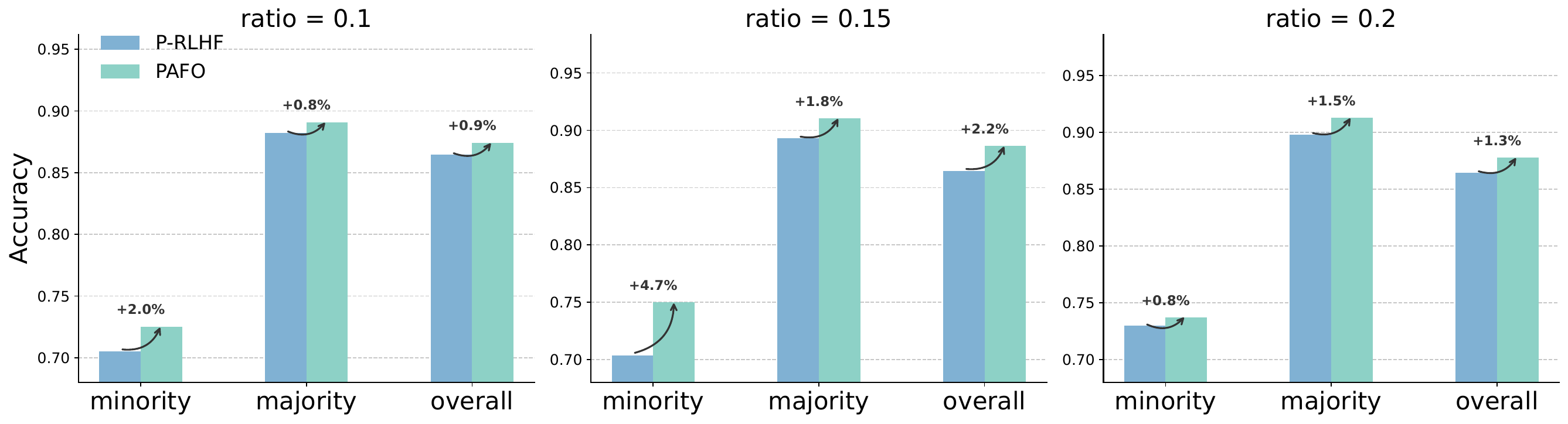}
    \caption{Sensitivity of PAFO to the minority ratio. P-RLHF is shown as a reference.}
    \label{fig:min_ratio}
    \vspace{-20pt}
\end{figure}

\paragraph{Sensitivity to Minority Ratio.} PAFO relies on a training-time partition of users into majority and minority groups to train the specialists used for distillation. We therefore study how sensitive PAFO is to the minority ratio used at this stage. As shown in Figure~\ref{fig:min_ratio}, we vary the ratio over \{0.10, 0.15, 0.20\}. PAFO achieves Pareto improvement over its personalization backbone (P-RLHF) at every setting, with Min Acc., Maj Acc., and Overall Acc. all rising consistently across the three ratios. This confirms that PAFO's advantage stems from the distillation mechanism itself and is robust to the choice of minority threshold.

\begin{wrapfigure}{r}{0.44\linewidth}
    \centering

    \includegraphics[width=\linewidth]{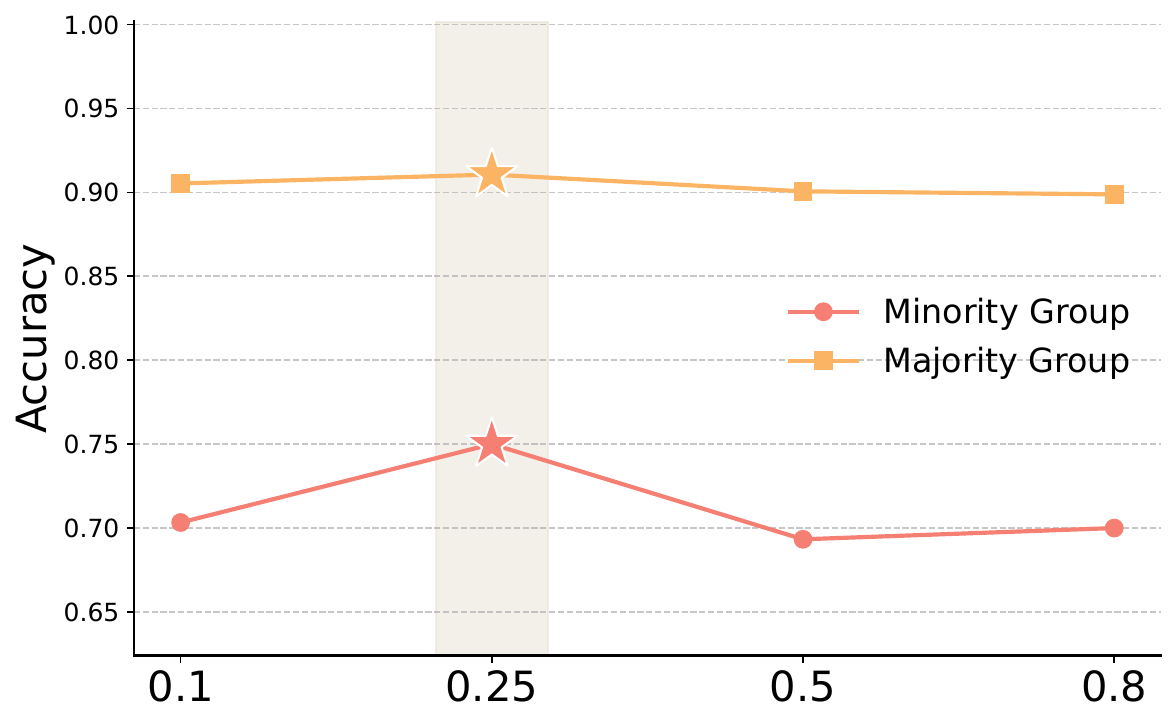}
    \caption{Effect of the distillation weight $\alpha$ ($\alpha \downarrow$, strength $\uparrow$).}
    \label{fig:loss_alpha}
\end{wrapfigure} 

\paragraph{Effects of Distillation Strength.} The coefficient $\alpha$ in Eq.~\ref{final loss} controls the trade-off between the distillation loss and the original BT reward modeling loss. We vary $\alpha \in \{0.1, 0.25, 0.5, 0.8\}$, 
with results shown in Figure~\ref{fig:loss_alpha}. 
The distillation strength has a much larger impact on Min Acc. than on Maj Acc.: Min Acc. increases first, then decreases as $\alpha$ increases, while Maj Acc. stays essentially flat across all settings. This asymmetry is intuitive---the majority is already well-supported by the base model, so it benefits little from additional distillation signal and is also not harmed by it; the minority, by contrast, depends almost entirely on the distillation signal to recover its group-specific preference structure, making it sensitive to how strongly that signal is weighted. 


%% file: figures/main_results.tex
\definecolor{pafoCol}{HTML}{eef2fd}   

\begin{table}[t]
  \caption{Performance comparison on Personal-LLM and DSP. We report utility metrics ($\uparrow$, higher is better) and fairness metrics ($\downarrow$, lower is better). Best results per row are in \textbf{bold}.}
  \label{tab:main_results}
  \centering
  \renewcommand{\arraystretch}{1.15}   
  \setlength{\tabcolsep}{1pt}
  \begin{tabular*}{\linewidth}{@{\extracolsep{\fill}} c c l ccccccc @{}}
    \toprule
    \textbf{Datasets} & \textbf{Group} & \textbf{Metric}
    & \textbf{BTL} & \textbf{VPL} & \textbf{GPO} & \textbf{Reg} & \textbf{Reweight} & \textbf{P-RLHF}
    & \cellcolor{pafoCol}\textbf{PAFO} \\
    \midrule
    \multirow{6}{*}{\begin{tabular}[c]{@{}c@{}}\textbf{Personal-}\\ \textbf{LLM}\end{tabular}}
    & \multirow{3}{*}{\textbf{Utility} $\uparrow$}
    & Acc.                & 83.20 & 81.25 & 80.60 & 87.25 & 87.60 & 86.45 & \cellcolor{pafoCol}\textbf{88.65} \\
    & & Min Acc.          & 56.67 & 50.33 & 34.00 & 71.33 & 73.33 & 70.33 & \cellcolor{pafoCol}\textbf{75.00} \\
    & & Maj Acc.          & 87.88 & 86.71 & 88.82 & 90.06 & 90.12 & 89.29 & \cellcolor{pafoCol}\textbf{91.06} \\
    \cmidrule(lr){2-10}
    & \multirow{3}{*}{\textbf{Fairness} $\downarrow$}
    & CV                 & 0.1924 & 0.2166 & 0.2862 & 0.1430 & 0.1357 & 0.1459 & \cellcolor{pafoCol}\textbf{0.1330} \\
    & & Slope            & 0.7707 & 0.7429 & 0.8009 & 0.4593 & 0.4672 & 0.5175 & \cellcolor{pafoCol}\textbf{0.4193} \\
    & & GINI             & 0.1021 & 0.1172 & 0.1426 & 0.0776 & 0.0743 & 0.0780 & \cellcolor{pafoCol}\textbf{0.0703} \\
    \midrule
    \multirow{6}{*}{\textbf{DSP}}
    & \multirow{3}{*}{\textbf{Utility} $\uparrow$}
    & Acc.                & 71.67 & 70.24 & 71.55 & 77.50 & 75.83 & 77.09 & \cellcolor{pafoCol}\textbf{78.08} \\
    & & Min Acc.          & 31.47 & 30.91 & 33.73 & 69.13 & 72.73 & 69.67 & \cellcolor{pafoCol}\textbf{74.47} \\
    & & Maj Acc.          & 78.76 & 77.18 & 78.22 & \textbf{78.98} & 76.38 & 78.40 & \cellcolor{pafoCol}78.72 \\
    \cmidrule(lr){2-10}
    & \multirow{3}{*}{\textbf{Fairness} $\downarrow$}
    & CV                 & 0.2738 & 0.2776 & 0.2582 & 0.1469 & 0.1373 & 0.1506 & \cellcolor{pafoCol}\textbf{0.1310} \\
    & & Slope            & 0.6780 & 0.6622 & 0.6414 & 0.2472 & \textbf{0.1685} & 0.2357 & \cellcolor{pafoCol}0.1696 \\
    & & GINI             & 0.1420 & 0.1440 & 0.1332 & 0.0796 & 0.0758 & 0.0823 & \cellcolor{pafoCol}\textbf{0.0728} \\
    \bottomrule
  \end{tabular*}
  \vspace{-15pt}
\end{table}

%% file: figures/ablation.tex
\begin{table}[t]
  \caption{Ablation of PAFO. Min Only: minority specialist alone. Maj $\rightarrow$ Min: Maj to Min distillation. Min + SFT: SFT replacing margin distillation. RI: relative improvement (\%) over Base.}
  \label{tab:ablation}
  \centering
  \setlength{\tabcolsep}{3pt} 
  \begin{tabular}{c cccc cccc}
    \toprule
    \multicolumn{1}{c}{\textbf{Datasets} ($\rightarrow$)}
    & \multicolumn{4}{c}{\textbf{Personal LLM}}
    & \multicolumn{4}{c}{\textbf{DSP}}\\
    \cmidrule(lr){2-5}
    \cmidrule(lr){6-9}
    \textbf{Methods} ($\downarrow$)
    & \textbf{Min Acc} & \textbf{RI}
    & \textbf{Maj Acc} & \textbf{RI}
    & \textbf{Min Acc} & \textbf{RI}
    & \textbf{Maj Acc} & \textbf{RI} \\
    \midrule
    \textbf{Base}
    & 70.33 & - & 89.29 & - & 69.67 & - & 78.40 & -\\

    \textbf{Min Only}
    & 86.67 & 23.23\textcolor{mygreen}{$\uparrow$} & 74.47 &  16.60\textcolor{myred}{$\downarrow$}
    & 82.73 & 18.75\textcolor{mygreen}{$\uparrow$} & 49.41 &  36.98\textcolor{myred}{$\downarrow$}\\

    \textbf{Maj$\rightarrow$ Min}
    & 71.33 & 1.42\textcolor{mygreen}{$\uparrow$} & 90.29 &  1.12\textcolor{mygreen}{$\uparrow$}
    & 69.60 & 0.47\textcolor{myred}{$\downarrow$} & 78.85 &  0.57\textcolor{mygreen}{$\uparrow$}\\

    \textbf{Min + SFT}
    & 73.67 & 4.75\textcolor{mygreen}{$\uparrow$} & 89.53 &  0.27\textcolor{mygreen}{$\uparrow$}
    & 67.40 & 3.26\textcolor{myred}{$\downarrow$} & 77.84 &  0.71\textcolor{myred}{$\downarrow$}\\

    \textbf{PAFO}
    & 75.00 & 6.64\textcolor{mygreen}{$\uparrow$} & 91.06 &  1.98\textcolor{mygreen}{$\uparrow$}
    & 74.47 & 6.89\textcolor{mygreen}{$\uparrow$} & 78.72 &  0.41\textcolor{mygreen}{$\uparrow$} \\
    \bottomrule
  \end{tabular}
\end{table}

%% file: sections/5_related_workv1.tex
\section{Related Work}

Our work intersects three lines of research. \textbf{Personalized reward modeling} learns user-dependent preference signals through explicit user representations, latent preference modeling, reward factorization, or low-rank decomposition \citep{zhao2025reinforced, poddar2024personalizing, bose2025lore, wang2026think, shenfeld2025language}, with some approaches operating at decoding or prompting time \citep{chen2024pad}. These methods primarily target average personalization quality and rarely audit utility distribution across users \citep{chakraborty2024maxmin}. \textbf{Fairness in preference learning} examines disparities induced by heterogeneous feedback. PRISM documents substantial individual- and group-level variation in alignment data \citep{kirk2024prism}; MaxMin-RLHF and Group Robust Preference Optimization improve worst-group utility via max-min or distributionally-robust objectives \citep{chakraborty2024maxmin, ramesh2024group}; \textcolor{blue}{P-GRPO enhances policy optimization for minority groups through group-wise advantage normalization in the GRPO learning process~\citep{P-GRPO}}; theoretical analyses caution against naive preference aggregation \citep{park2024rlhf, shirali2025direct} and recent works indicate that unfairness can originate at the reward-modeling stage \citep{ouyang2025towards, song2025towards}. 
\textcolor{blue}{Efforts have been made to fair reward modeling; however, they mainly address demographic fairness and bias mitigation~\cite{PAFO,ouyang2025towards}. In contrast, our work focuses on support-structured personalization unfairness, a distinct challenge that has received limited attention. Moreover, to our knowledge, existing approaches have not considered distillation-based methods for addressing this issue.}
\textbf{Multi-objective alignment}, including MODPO, DPA, RiC, MetaAligner, Panacea, ArmoRM, HaM, and PARM, optimizes over multiple preference dimensions or seeks Pareto-style trade-offs, but typically on \emph{pre-specified} objective axes rather than user-induced groups \citep{zhou2023beyond, wang2024arithmetic, yang2024rewards, yang2024metaaligner, zhong2024panacea, wang2024interpretable, mukherjee2024multi, lin2025parm}. PAFO differs from all three. Unlike personalized RM work, it explicitly targets utility \emph{fairness} across users. Unlike fairness-in-preference-learning work, it pursues \emph{Pareto improvement}, lifting minority utility without degrading majority utility, rather than max-min reallocation. Unlike multi-objective alignment, its objectives are induced by preference-support imbalance and require no group label at inference. A more comprehensive review of related work is provided in Appendix~\ref{related_work}.

%% file: sections/6_conclusionv1.tex
\section{Conclusion}

We studied personalized reward modeling through the lens of \emph{preference-support imbalance} and identified \emph{personalization unfairness}: a systematic disparity in reward modeling quality governed by how well a user's preference pattern is represented in training. We argued that mitigating this disparity should be framed as Pareto improvement over preference groups, not as a fairness--utility trade-off, and instantiated this view in PAFO, which trains group-specialized reward models and consolidates them into a single deployable model via conditional margin-level distillation. PAFO is the only method, among six baselines on two benchmarks, that simultaneously lifts minority-group accuracy and preserves majority-group accuracy while reducing user-level unfairness on three metrics; theoretically, its update direction approximates the implicit Pareto-optimal direction up to a covariance term. More broadly, our findings suggest that aggregate accuracy alone is an insufficient measure of personalization quality, and that user-level fairness deserves to be treated as a first-class objective in personalized alignment.

%% file: sections/7_appendix.tex
\section{Case Study}
To more intuitively illustrate the effect of PAFO, we further provide a case study on the scoring margins of a specific preference pair, as shown below. We select a test example from a minority-group user in the DSP dataset and present the scoring margins assigned by PAFO and the base model (P-RLHF) to the positive and negative candidate responses, respectively.

It can be clearly observed that the base model fails to recognize that the positive response, which exhibits a poetic style, is exactly the type of answer preferred by a user in the literature category, and instead assigns a relatively high score to the negative response, which is more restrained and less engaging in style. In contrast, PAFO is able to capture the user’s personalized preference more accurately, thereby making a distinction between the positive and negative responses that is better aligned with the user’s preference. This demonstrates its stronger capability in personalized reward modeling.

\begin{casestudybox}[Case Study]

\casefield{Prompt}
Name three things a person needs for a successful job hunt.

\casefield{User Preference}
literature

\vspace{0.5em}
\noindent\rule{\linewidth}{0.3pt}

\modelchoice{Base preferred response with margin = -3.7656}

To begin the search, you need to prepare yourself well, dress like a star,

Polish your resume and skills, polish and brush up, you need to stand out and go far.

You will need some patience too, as the journey may be long,

But persistence and determination will help you stay strong.

Networking is a key factor that should not be ignored,

Connections can open doors, that you may have ignored.

So in summary, preparation, patience, and networking can be,

The essential tools that you need for a successful job hunt spree.

\vspace{0.4em}
\noindent\rule{\linewidth}{0.3pt}

\casefield{PAFO preferred response with margin = 0.6562}

As an experienced researcher, I would say that three things a person needs for a successful job hunt are:

1. A clear understanding of their career goals and aspirations: Having clarity about the type of job, industry, company, and role they want to pursue will help individuals target their job search efforts and tailor their applications accordingly.

2. A professional and well-crafted resume/CV and cover letter: A well-prepared resume/CV and cover letter perfectly tailored to the job of interest can significantly increase the chances of getting shortlisted for an interview. Hence, it is important to invest time in crafting these documents.

3. Networking: While submitting online applications is an essential component of a job search, networking with industry professionals and building relationships with potential employers (through events, social media, etc.) can help individuals learn about job openings that may not be publicly advertised, and also get insider information about the company culture and job requirements.

\end{casestudybox}

\section{Prompts}
We use the following prompt template to concatenate the query x and the response y, and feed the resulting sequence to the model as the complete input.
\begin{promptbox}[Prompt Template]
"role": "user"

"content": "The following context contains preference demonstrations provided by the same user. Each example shows responses that this user considers acceptable or preferred for a given question.

Question:\{query\}

Liked Response:\{chosen response\}

Question:\{query\}

Liked Response:\{chosen response\}

Use the contexts above to generate a good response for the user prompt below. 

Now the question is:

\{query\}"

"role": "assistant"

"content": "\{response\}"
\end{promptbox}

\section{Experiment Details}
\label{details}

\subsection{Datasets Details}\label{appendix:data-details}

\textbf{Personal-LLM}. This is a controllable benchmark for simulating heterogeneous user preferences via reward model interpolation. 
This allows us to construct personalized preference pair data with quantifiable preference bias on top of this dataset.
Specifically, each prompt in the dataset is associated with eight candidate responses generated by advanced large language models such as GPT-4o. 
In addition, each response is scored by ten reward models that perform strongly on RewardBench. 
Therefore, given a prompt x and a candidate response y, we obtain a 10-dimensional reward vector $\mathbf{R}(x,y)\in \mathbb{R}^{10}$.

To simulate diverse user preferences, we sample a 10-dimensional user vector $\mathbf{u}$ from a Dirichlet distribution. 
This vector represents the user-specific preference weights over the ten reward models.
We then compute $\mathbf{u}^T \mathbf{R}(x,y)$ as the personalized score of the user for each prompt-response pair. 
For each prompt, the highest-scoring response among the eight candidates is selected as the positive sample, while the lowest is selected as the negative sample, thereby forming a personalized preference pair for that user. 

Finally, we simulate 100 users, with 20 samples for each user. 
To reduce randomness in user-level evaluation, we set the test set to have the same size as the training set, i.e., each user has 20 training samples and 20 test samples. 
To construct a dataset with stronger preference bias, we set the concentration parameter of the Dirichlet distribution used to simulate user vectors to \(\alpha=0.001\).

\textbf{DSP}. This is a domain-specific preference dataset containing multiple stylistic response preferences. 
Specifically, for each prompt, there are four stylistically different responses generated by GPT models, together with one original response without any role-play setting. 
We assume that different users correspond to one of five preference types. 
For each prompt, the response matching the user’s preference type is treated as the positive sample, while one response is randomly selected from the remaining four as the negative sample. Specifically, we simulate 500 users, and each user is also assigned 20 samples.



\subsection{Computation of User Support Rate}\label{app:support_rate}

\paragraph{Support Rate in Personal-LLM.}
For each simulated user \(u\), we first construct personalized preference pairs according to the user vector. 
Let \(\mathcal{D}_u=\{(x_i, y_{i,u}^{+}, y_{i,u}^{-})\}_{i=1}^{N_u}\) denote the preference dataset of user \(u\), where \(y_{i,u}^{+}\) and \(y_{i,u}^{-}\) are respectively the highest-scoring and lowest-scoring responses under user \(u\)'s personalized reward function. 
Given the user vector \(\mathbf{u}\) and the reward vector \(\mathbf{R}(x,y)\), the personalized score of user \(u\) for a prompt-response pair \((x,y)\) is defined as
\[
s_u(x,y) = \mathbf{u}^{\top}\mathbf{R}(x,y).
\]

To measure how common user \(u\)'s preferences are among the overall user population, we evaluate whether other users agree with user \(u\)'s preference pairs. 
For another user \(v \neq u\), we define the agreement indicator on the \(i\)-th preference pair of user \(u\) as
\[
A_{v}^{(u,i)}
=
\mathbb{I}
\left[
s_v(x_i, y_{i,u}^{+}) > s_v(x_i, y_{i,u}^{-})
\right],
\]
where \(\mathbb{I}[\cdot]\) is the indicator function. 
This indicator equals 1 if user \(v\) assigns a higher score to user \(u\)'s positive response than to the negative response, and 0 otherwise.

The support rate of user \(u\) is then computed as the average agreement from all other users over all preference pairs:
\[
\mathrm{Supp}(u)
=
\frac{1}{N_u(|\mathcal{U}|-1)}
\sum_{i=1}^{N_u}
\sum_{v \in \mathcal{U}, v \neq u}
A_{v}^{(u,i)},
\]
where \(\mathcal{U}\) denotes the set of all simulated users. 
A lower support rate indicates that the preferences of user \(u\) are less aligned with the majority of users and are therefore more likely to represent rare or outlier preferences. 
In our experiments, users with the lowest 15\% support rates are categorized as the minority group, while the remaining users are categorized as the majority group.

\paragraph{Support Rate in DSP.}
In DSP, each user is assigned one of several predefined preference types. 
Since users with the same preference type share the same preferred response style, the support rate of a user is directly defined by the population proportion of the corresponding preference type. 
Formally, let \(c(u)\) denote the preference type of user \(u\), and let \(\mathcal{U}_{c(u)}\) denote the set of users whose preference type is \(c(u)\). 
The support rate of user \(u\) is computed as
\[
\mathrm{Supp}(u)
=
\frac{|\mathcal{U}_{c(u)}|}{|\mathcal{U}|}.
\]
Thus, users belonging to less frequent preference types have lower support rates. 
In our experiments, three preference types (entertainment, literature, normal) are each assigned a proportion of 5\% and are treated as minority groups, while the remaining two preference types (business, academy) are assigned proportions of 25\% and 60\%, respectively, and are treated as majority groups.

\subsection{Dataset Bias Settings}\label{appendix:databias} 
To construct group-biased data distributions, we set the concentration parameter of the Dirichlet distribution used to simulate user vectors to \(\alpha=0.001\). 
We then use user support rates to partition users into majority and minority groups.
In Personal-LLM, the 15\% of users with the lowest support rates are defined as the minority group, while the remaining users are treated as the majority group.

In DSP, we construct biased data distributions by controlling the proportions of different preference types. 
Specifically, three preference types are each assigned a user proportion of 5\% and are treated as minority groups, while the remaining two preference types are assigned proportions of 25\% and 60\%, respectively, and are treated as majority groups. 
In this dataset, the user support rate directly corresponds to the population proportion of the associated preference type. More dataset details are provided in Appendix~\ref{app:support_rate}.

\subsection{Implementation Details} \label{app:baseline_details}
\paragraph{Baseline Details.}
For the Reweighting baseline, we assign a larger training weight to minority-group samples in the original personalized reward modeling objective. Formally, the training objective is defined as
\[
\mathcal{L}_{\mathrm{reweight}}
=
\mathbb{E}_{(x,p,y^+,y^-)\sim \mathcal{D}}
\left[
w_g \cdot \mathcal{L}_{\mathrm{BT}}(x,p,y^+,y^-)
\right],
\]
where \(w_g\) denotes the group-specific sample weight. For majority-group samples, we set \(w_g=1\), while for minority-group samples, we search \(w_g\) over \(\{1.1, 1.2, 1.3\}\).

For the Regularization baseline, we add an additional regularization term to the original personalized reward modeling objective to reduce performance disparities across users. The objective is defined as
\[
\mathcal{L}_{\mathrm{reg}}
=
\mathcal{L}_{\mathrm{BT}}
+
\lambda_{\mathrm{reg}}\mathcal{L}_{\mathrm{fair}},
\]
where \(\mathcal{L}_{\mathrm{fair}}\) denotes the regularization term over user reward distributions, and \(\lambda_{\mathrm{reg}}\) controls the strength of the regularization. We search \(\lambda_{\mathrm{reg}}\) over \(\{0.01, 0.001, 0.0001\}\).

Since Pareto fairness requires improving disadvantaged users without substantially degrading advantaged users, we select hyperparameters for Reweighting and Regularization according to the majority-group performance under comparable fairness improvement. Specifically, among the searched hyperparameters, we report the result that achieves the best majority-group accuracy, or equivalently the smallest majority-group degradation, while still improving fairness-related metrics.

\paragraph{LLM Backbone and Hyper-parameters.}
We use the open-source Qwen3-4B\citep{yang2025qwen3} as the backbone model.
During training, we adopt low-rank adaptation (LoRA)\citep{hu2022lora} to train our models, with LoRA alpha of 16, LoRA rank of 8, and LoRA dropout of 0.1.
We use a learning rate of \(5\times 10^{-4}\) for all the models.
On the Personal-LLM dataset, the train epochs is set to 2, while on the DSP dataset, the train epochs is set to 1.
For our framework, the parameter $\alpha$, which controls the trade-off between the distillation loss and the original reward modeling loss, is tuned in the range \(\{0.1, 0.25, 0.5, 0.8\}\).
All experiments are conducted on NVIDIA A100 GPUs with 80GB GPU memory.

\subsection{Computation of Fairness Metrics} \label{app:metrics}
The utility metrics, including \textit{Overall Accuracy}, \textit{Minority Accuracy} (Min Acc), and \textit{Majority Accuracy} (Maj Acc), are straightforward. Here, we mainly present the formulations of the fairness metrics.

\paragraph{Coefficient of Variation.}
We compute the coefficient of variation over user-level accuracies to measure the relative dispersion of model performance across users. Let \(a_u\) denote the accuracy of user \(u\), and let \(\mathcal{U}\) denote the full user set. The mean and standard deviation of user-level accuracies are defined as
\[
\mu_a = \frac{1}{|\mathcal{U}|}\sum_{u\in\mathcal{U}} a_u,
\]
\[
\sigma_a =
\sqrt{
\frac{1}{|\mathcal{U}|}
\sum_{u\in\mathcal{U}}
(a_u-\mu_a)^2
}.
\]
The coefficient of variation is then computed as
\[
\mathrm{CV} = \frac{\sigma_a}{\mu_a}.
\]
A smaller CV indicates more stable performance across users.

\paragraph{Gini Coefficient.}
We use the Gini coefficient to measure inequality in user-level accuracies. Given user-level accuracies \(\{a_u\}_{u\in\mathcal{U}}\), the Gini coefficient is computed as
\[
\mathrm{Gini}
=
\frac{
\sum_{u\in\mathcal{U}}
\sum_{v\in\mathcal{U}}
|a_u-a_v|
}{
2|\mathcal{U}|^2 \mu_a
}.
\]
A smaller Gini coefficient indicates lower inequality in user-level performance.

\paragraph{Accuracy-Support Rate Slope.}
To measure whether model performance depends on how common a user's preference is, we fit a linear regression between user support rate and user-level accuracy:
\[
a_u = \beta_0 + \beta_1 \mathrm{Supp}(u) + \epsilon_u,
\]
where \(\mathrm{Supp}(u)\) denotes the support rate of user \(u\). The slope \(\beta_1\) is estimated by ordinary least squares:
\[
\beta_1
=
\frac{
\sum_{u\in\mathcal{U}}
(\mathrm{Supp}(u)-\overline{\mathrm{Supp}})
(a_u-\mu_a)
}{
\sum_{u\in\mathcal{U}}
(\mathrm{Supp}(u)-\overline{\mathrm{Supp}})^2
},
\]
where
\[
\overline{\mathrm{Supp}}
=
\frac{1}{|\mathcal{U}|}
\sum_{u\in\mathcal{U}}
\mathrm{Supp}(u).
\]
A smaller slope indicates that model performance is less dependent on preference support rate, suggesting better user-level fairness.

\section{Limitations}
\label{limit}
This work still has several limitations. 
For example, our evaluation mainly relies on automatic metrics, and we do not further conduct a human audit of the unfairness produced by personalized reward models. 
Due to limited annotation resources and practical experimental constraints, we are not yet able to systematically analyze the behavioral differences of the model across different user groups through human evaluation. 
Although automatic metrics provide a scalable, controllable, and reproducible evaluation protocol, human audit may still offer more fine-grained evidence for understanding how model unfairness is perceived and manifested in real user-facing scenarios.
In addition, this work assumes that user preferences remain relatively stable during training and evaluation. 
However, in real-world personalization scenarios, user preferences may change over time, task contexts, or interaction stages, and the same user may exhibit different preference patterns at different stages. 
The current framework does not explicitly model such dynamic preference changes. 
Future work can further study Pareto-aligned personalized reward modeling under dynamic user preferences, so that user representations and fairness optimization objectives can be updated as preferences evolve.

\section{Proof of the Pareto-Aligned Direction}
\label{app:proof}

\begin{proof}[Proof of Theorem 1:]
We define the margin discrepancy as $\omega_g(i) = \sigma(m_g(i)) - \sigma(m_S(i))$, the total negative gradient of the distillation loss is
\begin{equation}
    v_t = \sum_{g=1}^n \mathbb{E}_{i \sim \mathcal{D}_g} \left[ \omega_g(i) \nabla_{\theta_S} m_S(i) \right].
\end{equation}
Let $\nabla \ell_S(i)=(1 - \sigma(m_S(i))) \nabla_{\theta_S} m_S(i)$, the true sample-level utility gradient on group $g$ is $\nabla U_{\theta_S}(g) = \mathbb{E}_{i \sim \mathcal{D}_g} [\nabla \log \sigma(m_S(i))]= \mathbb{E}_{i \sim \mathcal{D}_g}[\nabla \ell_S(i)]$. We can factorize the residual as $\omega_g(i) = \alpha_g(i) (1 - \sigma(m_S(i)))$, where $\alpha_g(i) = \frac{\sigma(m_g(i)) - \sigma(m_S(i))}{1 - \sigma(m_S(i))}$.
By applying the exact statistical identity for the expectation of a product, we decompose the gradient:
\begin{align}
    \mathbb{E}_{i \sim \mathcal{D}_g} \left[\omega_g(i) \nabla_{\theta_S} m_S(i)  \right] &= \mathbb{E}_{\mathcal{D}_g} [\alpha_g(i)] \mathbb{E}_{\mathcal{D}_g} [\nabla \ell_S(i)] + \text{Cov}_{\mathcal{D}_g}(\alpha_g(i), \nabla \ell_S(i)) \\
    &= \lambda_g \nabla U_{\theta_S}(g) + e_g
\end{align}
where $\lambda_g = \mathbb{E}_{\mathcal{D}_g}[\alpha_g(i)]$ is the global expected multiplier for group $g$. In addition, we have $\sigma(m_g(i)) \ge \sigma(m_S(i))$, hence $\alpha_g(i) \ge 0$, which guarantees $\lambda_g \ge 0$. 
Summing over all groups yields $v_t = \sum_{g=1}^n \lambda_g \nabla U_{\theta_S}(g) + \mathcal{E}_{cov}$, where $\mathcal{E}_{cov} = \sum_{g=1}^n e_g$.    
\end{proof}

\section{Representation Clustering Analysis}
\begin{figure}[htbp]
    \centering

    \begin{minipage}[t]{0.4\linewidth}
        \centering
        \includegraphics[width=\linewidth]{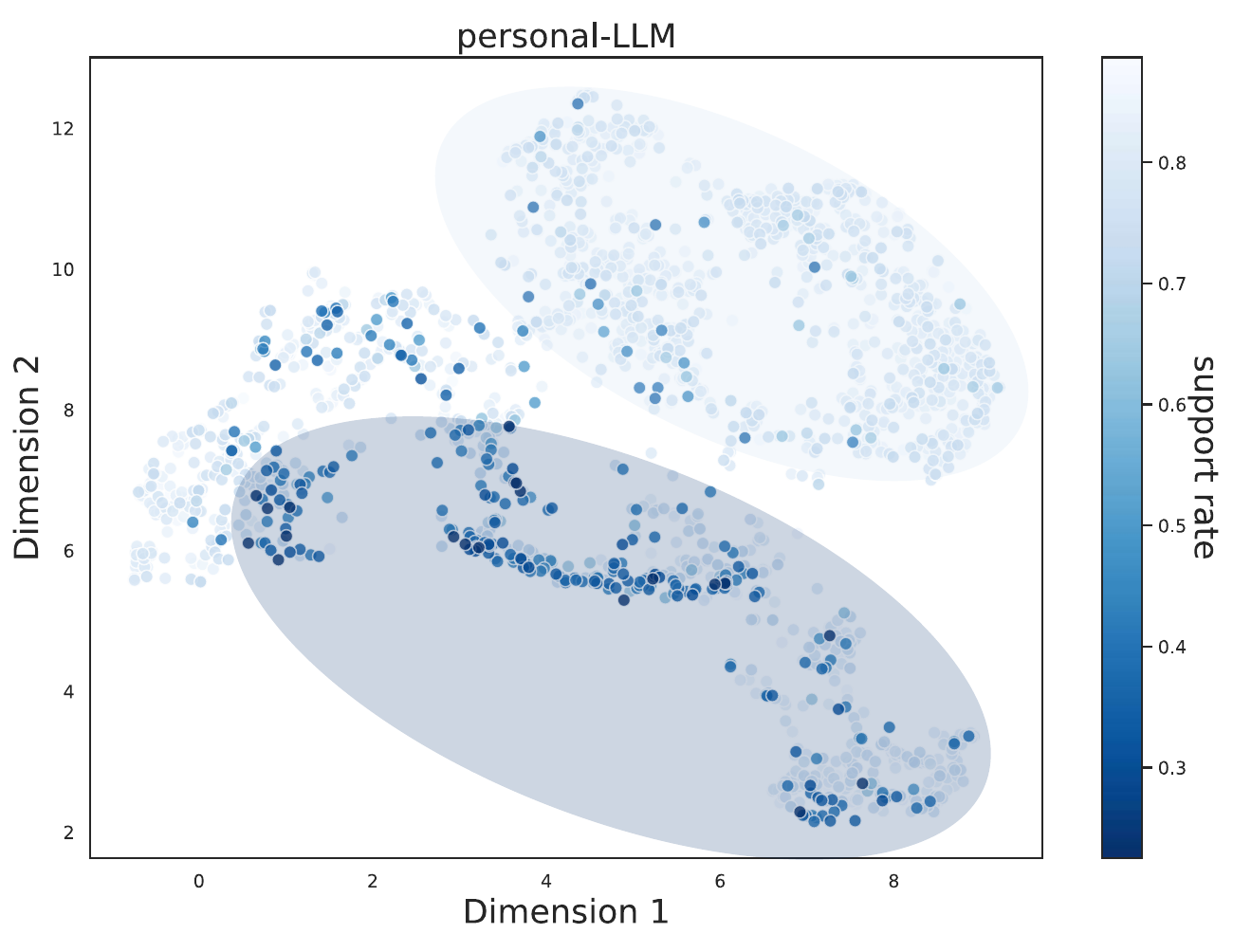}
        \caption{Visualization of hidden state on personal-LLM}
        \label{fig:cluster_personal}
    \end{minipage}
    \begin{minipage}[t]{0.4\linewidth}
        \centering
        \includegraphics[width=\linewidth]{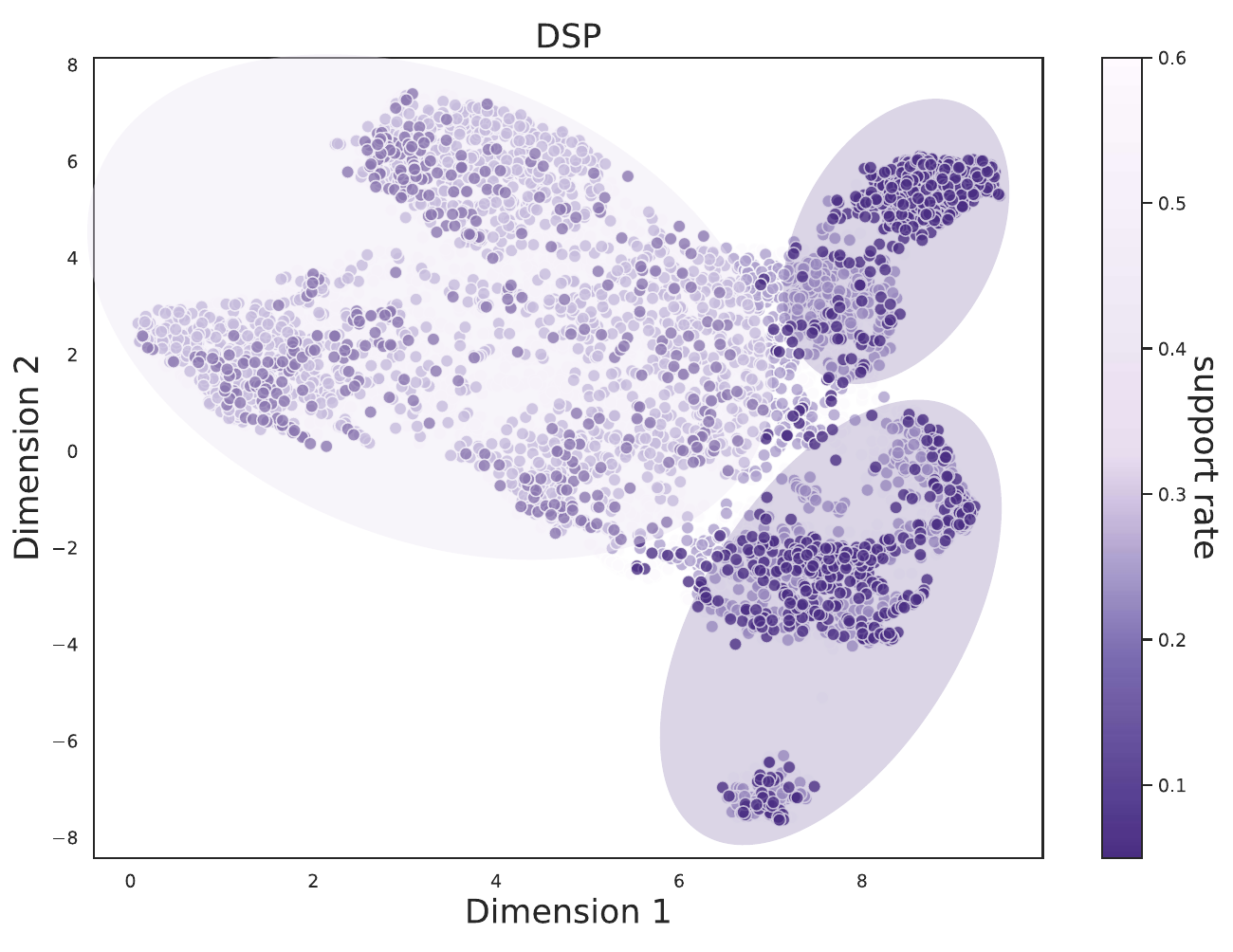}
        \caption{Visualization of hidden state on DSP}
        \label{fig:cluster_dsp}
    \end{minipage}

\end{figure}
To provide a qualitative view of preference heterogeneity, we perform a representation clustering analysis on the training data from both datasets. 
Specifically, we feed all training samples into the Qwen3-4B backbone model and extract their final-layer hidden representations. 
We then apply UMAP to project these representations into a two-dimensional space, as shown in Figure~\ref{fig:cluster_personal} and Figure~\ref{fig:cluster_dsp}.

The visualization shows that samples associated with different support rates exhibit non-uniform representation patterns. 
In particular, samples from low-support users are more likely to appear in separated regions or outlying clusters. 
Although this analysis is only qualitative, it suggests that low-support users may correspond to distinct preference patterns at the representation level. 
This observation supports our motivation for introducing group-specialized modeling and conditional distillation in PAFO, which aim to better capture preference structures that may be underrepresented in the mixed training distribution.

\begin{figure}[t]
    \centering
    \includegraphics[width=0.55\linewidth]{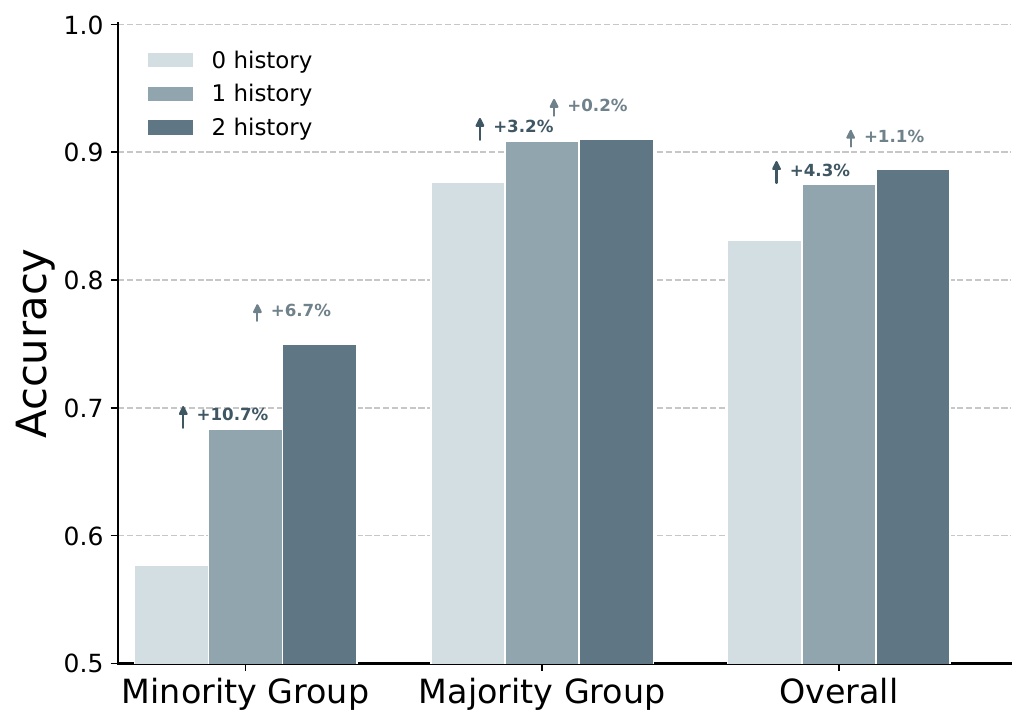}
    \caption{PAFO's performance with different numbers of historical examples on inputs.}
    \label{fig:his_num}
\end{figure}


\section{The Performance with Different History Length}
To verify the role of historical interaction data in personalized reward modeling, we vary the number of historical examples included in the input and observe the resulting changes in model performance, as shown in Figure~\ref{fig:his_num}. 
The results show that as the number of historical examples increases, the overall performance of PAFO continues to improve, confirming the importance of historical behavior information for user preference modeling.
Further analysis shows that the performance gain from increasing the number of historical examples from 0 to 1 is significantly larger than the gain from increasing it from 1 to 2. 
These results show that user history provides an effective personalization signal, and PAFO can exploit even limited historical information to better model heterogeneous preferences.

\section{Additional Related Work}\label{related_work}
\paragraph{Personalized Reward Model.}
Personalized reward modeling aims to learn user-dependent preference signals from heterogeneous user feedback, rather than fitting a single reward function to all users \citep{zhao2025reinforced,poddar2024personalizing}. Existing studies incorporate user information into the RLHF pipeline through explicit user representations or latent preference modeling, enabling individualized alignment  \citep{zhao2026nextquill, poddar2024personalizing}. Other methods further exploit shared structures across users to improve scalability and few-shot adaptation, such as personalized reward spaces, reward factorization, and low-rank reward modeling  \citep{wang2026think,bose2025lore,shenfeld2025language}. In addition, some approaches perform personalized reward estimation at the inference or prompting stage through decoding-time reward guidance or persona-guided prompting \citep{chen2024pad,ryan2025synthesizeme}. These studies demonstrate the importance of modeling preference heterogeneity for LLM personalization \citep{qiu2025measuring}. However, they primarily focus on average personalization quality or adaptation efficiency \citep{chakraborty2024maxmin}, while paying limited attention to whether utility is fairly distributed across different users or user groups.

\paragraph{Fairness in Preference Learning.}
When user preferences are highly heterogeneous, preference learning based on aggregate objectives may favor mainstream preferences, leading to utility disparities across user groups \citep{chakraborty2024maxmin, ramesh2024group, park2024rlhf, shirali2025direct,zhang2025fair,shi2024fair}. Datasets such as PRISM reveal substantial individual- and group-level variation in alignment feedback, highlighting the importance of whose preferences are represented during alignment \citep{kirk2024prism}. From an optimization perspective, MaxMin-RLHF and Group Robust Preference Optimization improve the performance of under-served or worst-performing groups through max-min or robust objectives \citep{chakraborty2024maxmin, ramesh2024group}. Other work studies heterogeneous feedback from a theoretical perspective, discussing the relationship between personalization and preference aggregation and showing that naively averaging user preferences can be problematic  \citep{park2024rlhf, shirali2025direct}. More directly related to reward learning, reward fairness regularization and reward-model fairness benchmarks further suggest that unfairness may already arise during reward modeling, before being amplified by subsequent policy optimization \citep{ouyang2025towards, song2025towards}. In contrast, we introduce Pareto fairness into personalized reward modeling, with the goal of improving the utility of minority or under-served preference patterns without degrading the performance of other user groups.

\paragraph{Multi-Objective Alignment.}
Multi-objective alignment provides relevant methodological background for modeling complex and potentially conflicting human preferences \citep{zhou2023beyond, wang2024arithmetic, zhong2024panacea, wang2024interpretable}. Prior work formulates LLM alignment as optimization over multiple preference dimensions \citep{zhou2023beyond, wang2024arithmetic}. For example, MODPO, DPA, RiC, MetaAligner, and Panacea study multi-objective preference modeling from the perspectives of multi-objective DPO, reward-space control, dynamic preference adjustment, and Pareto-style adaptation \citep{zhou2023beyond, wang2024arithmetic, yang2024rewards, yang2024metaaligner, zhong2024panacea}. On the reward-modeling side, methods such as ArmoRM, HaM, and PARM further investigate multi-dimensional reward decomposition, hypervolume maximization, and test-time preference-aware reward guidance \citep{wang2024interpretable, mukherjee2024multi, lin2025parm}. These works are conceptually related to ours, but most of them focus on trade-offs among predefined objective dimensions. Our work instead applies the Pareto perspective to group-level utility fairness in personalized reward modeling, aiming to improve reward model fairness under heterogeneous user preferences without requiring explicit group labels at inference time.

\section{Broader Impacts.}
\label{broader impacts}
This work aims to improve the fairness and reliability of personalized reward modeling for large language models. Its potential positive societal impact lies in reducing the performance gap between users with common preference patterns and users whose preferences are less represented in the training distribution. By encouraging Pareto-style improvement, PAFO seeks to improve the utility of under-served preference groups without degrading the performance of other groups, which may contribute to more inclusive and equitable personalized alignment systems. At the same time, this work also has potential negative societal implications. If the training data contain biased, noisy, or harmful preference signals, a personalized reward model may learn and preserve such biases more effectively. In addition, inaccurate group construction during training may lead to misleading fairness conclusions or insufficient protection for truly under-served users. More broadly, improved personalization techniques could be misused to optimize model behavior toward manipulative, polarizing, or otherwise harmful user-specific objectives. These risks suggest that personalized reward models should be deployed with careful data curation, privacy protection, auditing of group-level performance, and monitoring for unintended or malicious uses.

\clearpage